\documentclass{article}
\usepackage[T1]{fontenc}
\usepackage{stride_preprint,times}
\iclrfinalcopy
\usepackage{amsmath,amssymb,amsthm,mathtools}
\usepackage{graphicx,booktabs,multirow,array,longtable,float}
\usepackage{microtype}
\usepackage{hyperref,url}
\usepackage{newunicodechar,placeins}
\newunicodechar{–}{\textendash}
\hypersetup{colorlinks=true,linkcolor=blue,citecolor=blue,urlcolor=blue,pdftitle={STRIDE: State-Transition Representation via Increment Dynamics and Evolution},pdfauthor={Yuchen Xiong, Siming Huang, Jianfeng Sun}}

\newcommand{\R}{\mathbb{R}}

\newcommand{\ind}{\mathbf{1}}

\newcommand{\LL}{\textsc{STRIDE}}
\newcommand{\ES}{\operatorname{ES}}
\newcommand{\diag}{\operatorname{diag}}
\newcommand{\vecop}{\operatorname{vec}}

\newtheorem{proposition}{Proposition}
\title{STRIDE: State-Transition Representation via\\Increment Dynamics and Evolution}
\author{Yuchen Xiong\textsuperscript{1}\quad Siming Huang\textsuperscript{2}\quad Jianfeng Sun\textsuperscript{2}\\[0.6em]
\normalfont\small\textsuperscript{1}Xiamen University Malaysia\\
\normalfont\small\textsuperscript{2}The Hong Kong University of Science and Technology (Guangzhou)}

\begin{document}
\maketitle
\begin{abstract}

We introduce \LL\ (State-Transition Representation via Increment Dynamics and Evolution), a probabilistic forecasting framework for systems revisiting familiar local motions whose transition balance shifts over time. STRIDE clusters $n$-th-order derivatives into discrete states, aggregates sliding-window pair transitions via convolution, and takes lagged differences as high-dimensional local change rates. Sparse dynamics with lags and memory jointly forecast the leading coordinates of this trajectory obtained by proper orthogonal decomposition. Recombining these forecasts gives future increments, which are added to historical frequencies to recover future windows. A causal decoder estimates stepwise state-pair probabilities with consistent marginals, and each sampled transition selects a local continuous predictor. Across three seeds, we compare STRIDE with 14 baselines on 238 settings spanning 230 tasks from nine benchmark families and eight additional dataset settings with 432 test origins. STRIDE achieves lower whole-horizon Energy in 129/227 comparisons with VAR, 96/120 with direct SINDy, and 58/86 with DLinear. Against five state-space baselines, it wins 147 to 164 of 228 comparisons. On dysts, STRIDE scores 0.3838 versus 0.6202 for Markov--Gaussian. On MDBench excluding Allee, its median system-level Energy ratio to VAR is 0.582. Across the eight additional datasets, whole-horizon Energy falls by 7.90--45.14\% against DLinear on all eight and by 3.64--63.72\% against DeepAR on seven. In a separate comparison across 20 settings, it wins 18/19 against DS$^3$M, 17/20 against GDM, and 18/20 against ARHMM. The matched-control study covers 243 settings. Across 21 source groups with little variation in transition frequencies, STRIDE's whole-horizon mean Energy is within 1\% of each control. Among the 56 multi-step settings with high training-window variation ($A\ge 0.75$), 45 (80.4\%) improve by at least 5\% on a stability measure relative to at least one control. STRIDE improves forecast accuracy on dynamical systems such as Aizawa, Rössler, and Sprott A and on real-world time series such as ETTh1 and LOOP Seattle.

\end{abstract}

\section{Introduction}
\label{sec:introduction}

Many physical and real-world time series revisit the same handful of local behaviors, such as steady growth, sharp turns, or rapid oscillation, even as the balance among those behaviors drifts over time. Switching models capture this structure by pairing discrete regime states with distinct local dynamics \citep{pmlr-v54-linderman17a,hamilton1989regime}. When states are defined by clustering local $n$-th-order derivative estimates (local velocity, curvature, and higher-order change), sequences often keep recognizable local motion laws inside each state while shifting how frequently states switch. In hourly solar generation data, for example, the share of steps remaining within one state averaged 94.6\% in the first quarter of observed history and fell to 63.6\% in the last. In coarse-sampled R\"ossler trajectories, the share for state $q_2$ fell from 62.59\% to 40.00\%. Can predicting how these local transition statistics evolve improve probabilistic forecasts of the continuous signal?

\begin{figure}[t]
\centering
\includegraphics[width=0.92\linewidth]{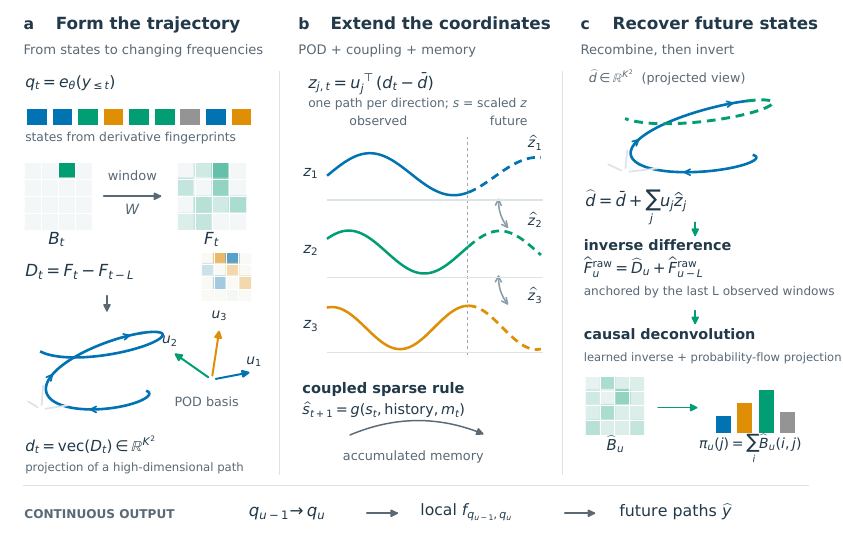}
\caption{\textbf{From observed transitions to continuous forecasts.} (a) Window convolution turns one-hot transitions \(B_t\) into frequencies \(F_t\), lagged differences \(D_t=F_t-F_{t-L}\) track local change rates, and POD splits vectorized \(D_t\) into directional sub-trajectories \(z_{j,t}\). (b) Standardized coordinates \(s_t\) advance jointly via coupled sparse dynamics with lags and memory. (c) Recombined rates yield future windows \(\widehat F_u^{\rm raw}\) by inverse differencing, causal deconvolution recovers stepwise pair probabilities \(\widehat B_u\) under probability-flow projection, and sampled pairs invoke local continuous predictors. Four states and three directions are shown.}
\label{fig:pipeline}
\end{figure}

Three existing lines of forecasting work each solve part of this problem without modeling how system-wide transition rates evolve jointly. Direct continuous forecasters, including autoregressive neural networks \citep{salinas2020deepar}, linear trend-seasonal maps \citep{zeng2023dlinear}, and scenario predictors \citep{timeprism2026}, map past values straight to future observations with no explicit coordinate for regime occupancy shifts. Hidden Markov and switching state-space models \citep{rabiner1989hmm,pmlr-v54-linderman17a,Xu_2026,gdm2025} separate regimes from continuous dynamics, yet evolve switches via a fixed chain or immediate one-step latent dependencies. Sliding-window methods re-estimating transition matrices from counts \citep{wilinski2019changing} track row-normalized levels per row, tying each fit to its baseline level and missing cross-transition trade-offs. None represents transition-frequency change rates as a joint trajectory that can be decomposed, forecast forward, and checked directly against realized windows.

We introduce \LL\ to forecast continuous dynamics through evolving state-transition change rates (Figure~\ref{fig:pipeline}). First, STRIDE clusters local $n$-th-order derivatives into $K$ discrete states representing distinct local motions (Figure~\ref{fig:pipeline}a). Because switching rates across a system shift together, a trailing window of length $W$ applies temporal convolution to observed state-pair indicators to compute the joint frequency $F_t$, and a lagged difference $D_t = F_t - F_{t-L}$ measures its local change rate. By entry--exit balance over the window, histories starting from different baselines share identical change rates whenever their shift matches, giving exact level-shift equivariance when forecast increments are added back to observed windows. Flattening $D_t$ places each step along a trajectory in $\mathbb{R}^{K^2}$, which POD splits into $r$ scalar directional sub-trajectories.

Next, STRIDE advances all $r$ directional sub-trajectories jointly via coupled sparse dynamics with finite lags and two-stage exponential memory, letting switching patterns interact while carrying recent momentum forward (Figure~\ref{fig:pipeline}b). Recombining predicted sub-trajectories restores future change rates $\widehat D_{t+h}$, and inverting the lag difference against the last $L$ observed windows yields future frequency windows $\widehat F_{t+h}^{\rm raw}$ (Figure~\ref{fig:pipeline}c, top).

Because each window $\widehat F_{t+h}^{\rm raw}$ averages $W$ transitions, recovering individual future states requires reversing that moving average (Figure~\ref{fig:pipeline}c, bottom). Direct algebraic recursion amplifies upstream forecast errors across steps, so STRIDE trains a causal convolutional decoder mapping forecast windows and observed boundaries directly to stepwise pair probabilities $\widehat B_h$. Row-wise simplex projection enforces probability-flow consistency across steps (Proposition~\ref{prop:flow}), giving time-varying transition matrices $T_{1:H}$ guided by memory. Sampling sequentially from $T_{1:H}$ yields the unique maximum-entropy state paths matching those decoded pair marginals, and each sampled transition $q_{u-1}\to q_u$ calls a local continuous predictor to form ensembles.

Separating transition-frequency dynamics from continuous rollout makes STRIDE's operating conditions testable. It suits systems where change rates are predictable at the sampling cadence, distinct ordered pairs drive different local predictors, history initializes windows and memory, and training data cover forecast regimes, as our Allee support experiment shows (Section~\ref{sec:late}). In low-activity regimes, matched controls stay within $\pm 1\%$ of context-only decoding in mean Energy (Section~\ref{sec:ablation}).

Our main contributions are threefold:
\begin{itemize}
\item \emph{Equivariant transition-increment representation.} We formulate regime drift via windowed local change rates of joint state-pair frequencies, using entry--exit balance and level-shift equivariance and isolating increment evolution from direct-frequency and history-only controls.
\item \emph{Coupled sub-trajectory dynamics and probability-flow deconvolution.} We decompose the change-rate trajectory via POD into coupled sub-trajectories with memory, invert the lag difference into future windows, and apply causal deconvolution with simplex projection (Proposition~\ref{prop:flow}), ensuring causal inputs and consistent pair marginals along sampled paths.
\item \emph{Multi-benchmark evaluation and matched controls.} Across 238 settings against 14 baselines, 20 switching comparisons, 243 matched controls, and a locked 64-trajectory Aizawa cohort, STRIDE improves accuracy and stability in active regimes; preserving temporal order in change rates versus lead shuffling cuts whole-horizon Energy by 15.0\% on coarse R\"ossler and 34.4\% on hourly solar generation.
\end{itemize}

\subsection{Related work}
\label{sec:related-work}

\paragraph{Switching and state-space models.}
Hidden Markov models \citep{rabiner1989hmm} and Markov-switching autoregressions \citep{hamilton1989regime} pair discrete regimes with stationary emissions or local linear laws. Recurrent switching linear dynamical systems (rSLDS; \citealp{pmlr-v54-linderman17a}) condition switches on continuous latent states; neural extensions such as DS$^3$M \citep{Xu_2026} condition switching on recurrent history, and the Gumbel Dynamical Model \citep{gdm2025} blends local dynamics with Gumbel weights. Closer to empirical counting, \citet{wilinski2019changing} estimates time-varying transition matrices across sliding financial windows. These methods model transition laws through one-step conditional probabilities or raw levels. STRIDE tracks each ordered pair as a share of \emph{all} transitions in the window and forecasts lagged change rates of that joint table as a trajectory before inverting future windows and decoding stepwise probabilities.

\paragraph{Direct continuous and probabilistic forecasting.}
Direct forecasters predict signal values without explicit regime-frequency coordinates. Autoregressive recurrent networks such as DeepAR \citep{salinas2020deepar} estimate sequential univariate predictive distributions, attentional copulas such as TACTiS-2 \citep{ashok2024tactis2} model flexible multivariate dependencies, and TimePrism \citep{timeprism2026} outputs discrete scenario sets over complete paths. Simple baselines including trend-seasonal linear maps (DLinear; \citealp{zeng2023dlinear}) and retrieval-based context parroting \citep{zhang2026parroting} remain competitive alongside time-series foundation models \citep{boom}. We compare these predictors against STRIDE under identical splits and proper path scores while using fixed-output controls to isolate the frequency branch.

\paragraph{Dynamical coordinates, operators, and memory.}
Dynamical forecasters also differ in how they choose coordinates and represent memory. SINDy \citep{brunton2016sindy} and symbolic search \citep{sun2023spl} identify parsimonious equations in candidate libraries, while autoencoder and recurrent extensions co-learn latent coordinates with sparse dynamics \citep{champion2019coordinates,sindyshred2026}. Operator-learning approaches map nonlinear evolution into linear or spectral coordinates, including Koopman forecasters \citep{wang2023knf,zheng2025koonpro} and Fourier neural operators \citep{li2021fno,tran2023ffno}. While those methods infer latent observables from continuous states, STRIDE applies POD \citep{sirovich1987pod,holmes2012turbulence} to measurable pair-frequency change rates, splitting the high-dimensional trajectory into low-dimensional sub-trajectories evolved jointly using delay and exponential memory terms reminiscent of next-generation reservoir computing and Mori--Zwanzig closures \citep{gauthier2021nvar,lin2021mz}. Keeping every intermediate quantity tied to observed counts allows direct validation against realized transitions.

\section{Forecasting through transition-frequency dynamics}
\label{sec:method}

Let \(y_t\in\R^d\) denote observations and \(\mathcal H_t\) the information available at origin \(t\). We seek a distribution over \(y_{t+1:t+H}\), evaluated as an ensemble of complete paths. Every parameter is fitted on training data, validation selects hyperparameters and checkpoints, and future values serve solely as targets or held-out test references.

Following Figure~\ref{fig:pipeline}, the five subsections build \LL\ step by step: clustering derivative states and forming change-rate trajectories (Section~\ref{sec:lift}), decomposing directional sub-trajectories via POD (Section~\ref{sec:pod}), evolving coupled coordinates with lags and memory (Section~\ref{sec:dynamics}), inverting increments and decoding stepwise probabilities (Section~\ref{sec:inverse}), and routing local continuous predictors (Section~\ref{sec:writer}).

\subsection{Define states and form the transition-frequency trajectory}
\label{sec:lift}

The discrete state \(q_t\) groups local motion patterns defined from \(n\)-th-order derivative estimates. At each step \(t\), STRIDE fits a polynomial to the trailing \(J\) observations via the dataset's observation adapter (raw univariate values or a multivariate projection) and evaluates derivatives \(\widehat y_t^{(k)}\) at \(t\). Candidate fingerprints are \((\widehat y_t^{(1)},\ldots,\widehat y_t^{(4)})\) or \((y_t,\widehat y_t^{(1)},\widehat y_t^{(2)})\). Training medians and interquartile ranges normalize each block, geometric weights (\(1,0.45,0.45^2,\ldots\)) prevent noisier high-order derivatives from dominating, and training-only K-means defines \(K\) states, assigning \(q_t=e_\theta(y_{t-J+1:t})\) by nearest center.

Given the state sequence, each transition \((q_{t-1},q_t)=(i,j)\) produces a \(K\times K\) one-hot table \(B_t\) with entry \((i,j)=1\). Because switching rates move collectively across a system, sliding a trailing window of length \(W\) applies temporal convolution to give the joint frequency \(F_t\), and a lagged difference over lag \(L\) captures the local change rate:

\begin{equation}
B_t(i,j)=\ind\{q_{t-1}=i,q_t=j\},\qquad
F_t=\frac{1}{W}\sum_{k=0}^{W-1}B_{t-k},\qquad D_t=F_t-F_{t-L}.
\label{eq:lift}
\end{equation}

The convolution averages \(W\) transitions ending at \(t\) (over \(W+1\) states); three instances of \(1\to 2\) in ten steps yield \(F_t(1,2)=0.3\). Lag \(L\) separates the two windows, making \(D_t\) an unscaled zero-sum change rate (while entries of \(F_t\) sum to one). For fixed encoder and increment predictor, adding a constant table \(C\) to every frequency window leaves \(D_t\) unchanged. Equation~\eqref{eq:integration} then adds \(C\) to every reconstructed window through its observed anchor. This level-shift equivariance applies whenever the shifted tables remain valid joint-frequency tables. Flattening \(D_t\) into \(d_t=\vecop(D_t)\in\R^{K^2}\) assigns one coordinate per pair and turns successive shifts into a high-dimensional trajectory (\(D_t=0\) whenever frequencies repeat at lag \(L\)). The entry--exit identity \(D_t=W^{-1}\sum_{r=0}^{L-1}(B_{t-r}-B_{t-W-r})\) follows by telescoping the successive window differences: increments record the mass entering the window minus the mass leaving it.

\subsection{Resolve the trajectory into directional coordinate paths}
\label{sec:pod}

Because increases in one transition trade off against decreases in others within the zero-sum table \(D_t\), coordinates of \(d_t\) move together along a small set of dominant axes. POD decomposes \(d_t\) via principal component analysis on training increments to extract the orthonormal directions \(U_r=[u_1,\ldots,u_r]\in\R^{K^2\times r}\) of largest variance around the training mean \(\bar d\). Projecting \(d_t-\bar d\) onto \(U_r\) splits the trajectory into \(r\) scalar directional sub-trajectories \(z_{j,t}=u_j^\top(d_t-\bar d)\) tracking collective patterns of changing transition mass, which we standardize using training locations \(b\) and scales \(a\in\R_+^r\):

\begin{equation}
s_t=\diag(a)^{-1}[U_r^\top(d_t-\bar d)-b],\qquad
\widehat d_t=\bar d+U_r[b+\diag(a)\widehat s_t].
\label{eq:pod}
\end{equation}

Because training increments are zero-sum, the nonzero POD rank is at most \(K^2-1\). Ranks retaining 90\%, 95\%, or 99\% of training-sample variance form dataset-specific candidates; each candidate basis is refitted on all eligible training changes, and validation \(D\) forecasts over \(\min(H,16)\) steps narrow down rank and history length before full-chain validation Energy makes the final selection. Exchange selects rank 1 and StocBench rank 15.

\subsection{Advance the coordinate paths through coupling and memory}
\label{sec:dynamics}

Given the decomposed sub-trajectories \(s_t\), we follow SINDy \citep{brunton2016sindy} to couple and advance each directional sub-trajectory as a weighted sum of feature functions, pruning small coefficients during fitting:

\begin{equation}
\begin{aligned}
\widehat s_{t+1}&=(\Theta_t\Xi)^\top+\varepsilon_{t+1},\\
\Theta_t&=\big[1,s_t^\top,\sin(\pi s_t)^\top,\cos(\pi s_t)^\top,\\
&\qquad\{(s_{t-\ell}-s_t)^\top\}_{\ell=1}^{p-1},
\{(m_t^{(\rho,j)}-s_t)^\top\}_{\rho\in\mathcal R,\ j\in\{1,2\}}\big].
\end{aligned}
\label{eq:sindy}
\end{equation}

Here \(\mathcal R=\{0.6,0.9\}\), so each coordinate update \(\widehat s_{j,t+1}=\Theta_t\Xi_{:,j}+\varepsilon_{j,t+1}\) takes inputs across \emph{all} retained sub-trajectories: current positions and bounded sine/cosine nonlinearities, short-term differences across \(p-1\) lags capturing local shape, and gaps relative to fast (\(\rho=0.6\)) and slow (\(\rho=0.9\)) accumulated memories updated via two cascaded recursions:

\begin{equation}
m_t^{(\rho,1)}=\rho m_{t-1}^{(\rho,1)}+(1-\rho)s_t,\qquad
m_t^{(\rho,2)}=\rho m_{t-1}^{(\rho,2)}+(1-\rho)m_t^{(\rho,1)}.
\label{eq:memory}
\end{equation}

We estimate \(\Xi\) via scale-normalized ridge regression followed by three sequential threshold-and-refit passes, selecting lag \(p\) on validation and taking \(\varepsilon_{t+1}\) covariances from training residuals. At each step, Equation~\eqref{eq:sindy} advances all directional sub-trajectories jointly before updating lags and memory, and Equation~\eqref{eq:pod} recombines them through the fixed POD basis into \(\widehat D_{t+h}\), passing non-Markovian historical context into future transition schedules.

\subsection{Invert increments and decode transition probabilities}
\label{sec:inverse}

To recover future transition windows from the recombined change-rate trajectory \(\widehat D_{t+h}\), we invert the lagged difference using the last \(L\) observed frequency tables as boundary conditions:

\begin{equation}
\widehat F_{t+h}^{\rm raw}=\widehat D_{t+h}+
\begin{cases}F_{t+h-L},&h\le L,\\
\widehat F_{t+h-L}^{\rm raw},&h>L.
\end{cases}
\label{eq:integration}
\end{equation}

Because each reconstructed window \(\widehat F_u^{\rm raw}\) averages \(W\) transitions via temporal convolution, recovering individual stepwise state pairs requires inverting that convolution. With exact frequencies and boundary transitions, the moving average has the recursive algebraic inverse \(B_u=W F_u-\sum_{k=1}^{W-1}B_{u-k}\), which propagates and amplifies upstream forecast errors across steps when applied to predicted windows. \LL~therefore keeps inverse differencing fixed and learns \emph{causal deconvolution} to recover stepwise transition probabilities from forecast windows and observed boundaries.

The decoder is a small causal convolutional neural network (CNN). At forecast step \(h\), it uses frequency-ensemble means and standard deviations through that step, together with \(W-1\) observed boundary pairs, \(L\) observed frequency anchors, the last observed state \(q_t\), and horizon features. Restricting inputs to current and preceding steps guarantees causality. It has width 16 and three residual blocks with dilations 1, 2 and 4. With upstream components fixed, training minimizes \(\|\widehat B_h-Y_h\|_F^2\), where \(Y_h=e_{q_{t+h-1}}e_{q_{t+h}}^\top\) is the observed pair's \(K\times K\) one-hot matrix.

Let \(A_h\) be the CNN output after softmax over its \(K^2\) entries. The vector \(\pi_{h-1}\) gives state probabilities at the preceding step, starting from \(\pi_0=\operatorname{onehot}(q_t)\). The decoder projects row \(i\) onto \(\Delta(\pi_{h-1}(i))\), the set of nonnegative vectors summing to \(\pi_{h-1}(i)\):

\begin{equation}
\widehat B_h(i,:)=\operatorname{Proj}_{\Delta(\pi_{h-1}(i))}A_h(i,:),\qquad
\pi_h(j)=\sum_i\widehat B_h(i,j).
\label{eq:flow}
\end{equation}

Each table has nonnegative entries and total mass one, and its column sums give the next state probabilities that become the following table's row sums, a property we call \emph{probability-flow consistency}. On positive-mass rows, setting \(T_h(i,j)=\widehat B_h(i,j)/\pi_{h-1}(i)\) defines the stepwise conditional transition matrices (zero-mass rows are unreachable and can take any normalized vector).

\begin{proposition}[Probability-flow consistency]
\label{prop:flow}
If Equation~\eqref{eq:flow} holds and $q_t\sim\pi_0$, then sampling with $T_1,\ldots,T_H$ gives $\Pr(q_{t+h-1}=i,q_{t+h}=j\mid\mathcal H_t,\widehat B_{1:H})=\widehat B_h(i,j)$.
\end{proposition}
\begin{proof}
The row sums of \(\widehat B_h\) are \(\pi_{h-1}\), so
\(\Pr(q_{t+h-1}=i,q_{t+h}=j)=\pi_{h-1}(i)T_h(i,j)=\widehat B_h(i,j)\).
Its column sums give \(\pi_h\), completing the induction from \(\pi_0\).
\end{proof}

Conditional on the history-dependent sequence \(T_{1:H}\), sequential sampling under Proposition~\ref{prop:flow} yields the unique maximum-entropy path distribution matching decoded pair marginals \(\widehat B_{1:H}\), and averaging \(\widehat B_{1:H}\) with observed boundaries gives expected sampled window frequencies by linearity of expectation. The maximum-entropy claim follows from the entropy chain rule: conditioning on the whole generated prefix cannot increase conditional entropy relative to conditioning only on its last state. Sequential sampling attains the resulting bound for the fixed pair marginals.

\subsection{Route local continuous functions with the predicted states}
\label{sec:writer}

Sampling sequentially from \(T_1,\ldots,T_H\) yields an ordered transition \(q_{u-1}=i\to q_u=j\) at each step, calling a local model predicting a residual around constant-velocity extrapolation (\(2\widehat x_{u-1}-\widehat x_{u-2}\)). In training-standardized coordinates \(x\),

\begin{equation}
\widehat x_u=2\widehat x_{u-1}-\widehat x_{u-2}
+(\beta_0+\beta_{ij})^\top
\phi(\widehat x_{u-1},\widehat x_{u-1}-\widehat x_{u-2})+\eta_{u,ij}.
\label{eq:writer}
\end{equation}

The polynomial library uses validation-selected degree 1, 2, or 3. A shared ridge fit \(\beta_0\) precedes transition-specific residual fits \(\beta_{ij}\), with Gaussian innovation covariances from training residuals; Equation~\eqref{eq:writer} rolls out continuous forecasts recursively.

During rollout, sampled pairs \((q_{u-1},q_u)\) directly route local predictors fitted under encoder labels, distinct from post-hoc diagnostics that re-encode trailing windows of \(J\) generated observations. Fitting proceeds in stages (encoder and POD, sparse dynamics with validation-based selection, readout against observed transitions, and local predictors independently), allowing matched controls to swap any module while holding states and local functions fixed.

\section{Experiments}
\label{sec:experiments}

We organize our evaluation around dataset scope and protocol (Section~\ref{sec:word-3.1}), end-to-end accuracy across historical families and added cohorts (Section~\ref{sec:late}), increment dynamics and No-$D$ applicability (Section~\ref{sec:ablation}), downstream recovery and calibration (Section~\ref{sec:inverse-results}), and long-horizon attribution (Section~\ref{sec:long-horizon-control}).

\subsection{Protocol and dataset hierarchy}
\label{sec:word-3.1}

To keep denominators explicit, our study separates three cohorts. First, the \emph{historical benchmark} comprises 230 evaluated tasks (of 240 planned; 220 fully scorable across three fitting seeds) across nine families: dysts, MDBench, FinStressTS, ProbTS, MDGen-pilot, StocBench, GIFT-Eval, fev, and BOOM (200 multi-step plus 30 single-step FinStressTS tasks). Adding eight datasets (Weather, METR-LA, PeMS-BAY, PAMAP2, Duffing, DoubleGyre, TEP, SleepEDF; 432 test origins) yields the \emph{238-task main inventory} across 136 system/source groups (129 historical plus seven added, sharing Weather). Second, the \emph{243-setting matched-control cohort} adds two expanded-support studies (Allee, Kuramoto--Sivashinsky) and three independent-initial-condition cohorts (227 complete four-arm settings, 197 multi-step). Third, a \emph{20-setting switching comparison} evaluates specialized switching baselines.

Across the main inventory, 14 baselines span clustered Markov--Gaussian, continuous Gaussian Markov, Gaussian HMM, third-order HMM, ARHMM, direct SINDy, seasonal naive, Holt, VAR, NVAR, context parroting, DLinear, DeepAR, and TimePrism. Five state-space baselines cover all 238 tasks (228 complete pairs), whereas others run on eligible cohorts. Comparisons share splits, contexts, horizons, up to 256 origins per task, seeds (1729, 2718, 3141), 32 paths (point-mass scoring for deterministic models), and training-only scaling and POD; the same-$q$ Markov control keeps STRIDE's encoder and local functions. The primary metric is Euclidean path Energy \citep{gneiting2007proper,shahroudi2024trajectoryenergy,ashok2024tactis2}; for \(S>1\) paths and \(M\) observed coordinates,

\begin{equation}
\widehat\ES=\frac{1}{S\sqrt M}\sum_{s=1}^{S}\|x_s-y\|_2
-\frac{1}{2S(S-1)\sqrt M}\sum_{s\ne r}\|x_s-x_r\|_2.
\label{eq:energy}
\end{equation}

Intermediate metrics include \(D\) MSE, window \(F\) MSE, pair and state Brier scores, fair CRPS, variograms, and coverage \citep{ferro2014fair,scheuerer2015variogram}. Task scores average matched origins within seeds and three seeds equally; family means weight system/source groups equally. Bootstrap intervals are descriptive, and the independent Aizawa study locks fits before generating fresh paths.

\subsection{Benchmark accuracy across historical families and additional cohorts}
\label{sec:late}

Across the combined 238-task inventory, STRIDE lowers whole-horizon Energy in 129/227 complete pairs against VAR, 96/120 against direct SINDy, 58/86 against DLinear, and 147/228 to 164/228 against the five state-space baselines. On the 200 historical multi-step tasks (1,777 complete pairs of 1,903 planned; Figure~\ref{fig:late-comparisons}), system-weighted late win fractions reach 73.6--78.5\% against those five state-space baselines and 77.7\% whole / 77.3\% late against DLinear (44/72), while matching or edging out NVAR (51.6\%) and TimePrism (52.5\%).

\begin{figure}[t]
\centering\includegraphics[width=0.95\linewidth]{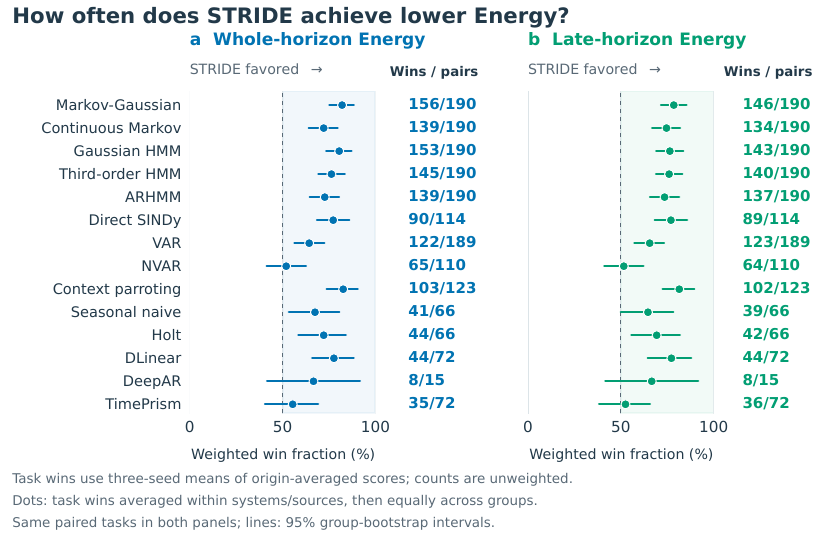}
\caption{\textbf{Whole- and late-horizon wins on the historical benchmark.} (a) Whole-path Energy; (b) final-half path Energy, on identical comparator-specific multi-step cohorts. Bold counts give unweighted task wins / paired tasks. Points average task win indicators within each system/source, then equally across groups; lines give descriptive 95\% group-bootstrap intervals. The dashed line marks 50\%. The five full-inventory state-space baselines each have 190 paired tasks.}
\label{fig:late-comparisons}
\end{figure}

Within the nine historical families, STRIDE outperforms baselines across both dynamical systems and real series. On \texttt{dysts} (21 conditions, eight chaotic systems), it lowers Energy to 0.3838 vs 0.6202 (Markov--Gaussian) and 0.6283 (HMM), beating all eleven comparators; on MDBench (excluding Allee), it reaches a 0.582 median Energy ratio $R_u$ vs VAR across 65 systems (winning 75/100 conditions, 48/65 systems). On real time-series families, STRIDE improves over Markov--Gaussian by 18.0\% on ProbTS (beating TimePrism 0.6617 vs 0.7384), beats DLinear (4.70--11.48\%) and TimePrism (14.23--44.48\%) across all four Weather horizons, cuts Energy 10.0\% vs VAR on GIFT-Eval (hospital: 33.21\% vs DLinear, 20.58\% vs TimePrism), beats all five state-space baselines and TimePrism on fev (2.159 vs 3.693/2.800; admissions 33.04\% vs DLinear), and stays below Markov--Gaussian, HMMs, DLinear, DeepAR, and TimePrism on BOOM (23/23).

Consistent gains extend to the eight added datasets (Figure~\ref{fig:new-eight-main}) and 20 switching comparisons. Frozen Full models beat DLinear on all eight datasets (7.90--45.14\%), DeepAR on seven (3.64--63.72\%), and TimePrism on five (Weather 20.20\%, DoubleGyre 23.95\%, SleepEDF 22.92\%). Across 20 switching settings, STRIDE wins 18/19 vs DS$^3$M, 17/20 vs GDM, and 18/20 vs ARHMM (medians 4.53\%, 3.67\%, 4.19\%): Sprott A preserves truer final-quarter increment directions, DoubleGyre cuts late Energy 62.79\% vs GDM, and late gains over DS$^3$M concentrate in higher-variation histories (TEP $-28.02\% \to +22.67\%$, METR-LA $-3.33\% \to +6.40\%$).

\begin{figure}[tbp]
\centering
\includegraphics[width=\linewidth]{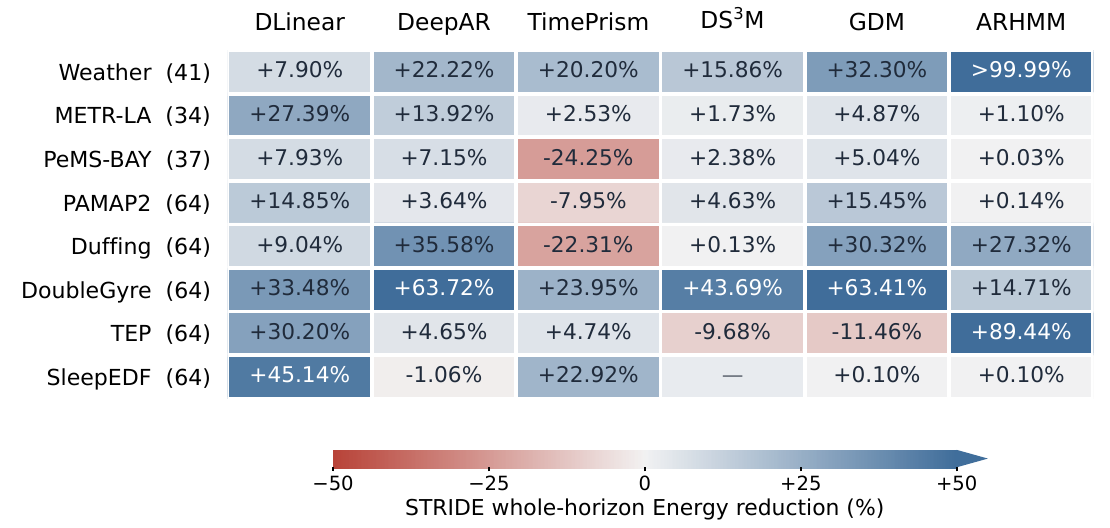}
\caption{\textbf{STRIDE versus six baselines on eight additional datasets.} Cells report the whole-horizon Energy reduction $100(E_{\rm baseline}-E_{\rm STRIDE})/E_{\rm baseline}$. Blue favors STRIDE; red favors the baseline. Dataset labels give the number of test cases, totaling 432. Each comparison uses three fitting seeds and the same contexts, horizon and origins. Colors saturate at $\pm50\%$; cell labels retain the computed reductions. DS$^3$M--SleepEDF is unscored because of full-history training cost.}
\label{fig:new-eight-main}
\end{figure}

Full-inventory means also pinpoint extrapolation limits: MDBench mean Energy rises to 25.458464 vs 0.223915 for VAR (0.3851 vs 0.1887 without Allee) because Allee training runs cover $6.36$--$11.49$ while targets reach $21.75$--$93.12$. Equal-budget support expansion (eight trajectories) lifts Allee completion from $207/480$ to $480/480$, cuts original-origin Energy by 97.20\% ($144.618249 \to 4.044027$, beating 7/9 refitted baselines at 0.683717), and lowers Kuramoto--Sivashinsky Energy by 81.52\% ($6.6726 \to 1.2329$, beating 10/11 baselines). Overall, 57,040/57,273 historical ensembles finish (99.593\%, 220/230 tasks), with ten tasks halting during continuous recursion at nonfinite or $\ge 10^8$ values while upstream transitions stay finite.

\subsection{Increment dynamics, \texorpdfstring{No-$D$}{No-D} ablations, and regime applicability}
\label{sec:ablation}

When transition frequencies barely drift, predicting changes in $D$ should leave baseline behavior undisturbed. True $D=0$ accounts for 98.81\% of Exchange, 80.67--85.41\% of Weather, and 85.43\% of Mackey--Glass evaluated targets, matching context-only single-step decoding within $\pm 1\%$ on Exchange and $\pm 3\%$ on Weather $H=336$. Across 243 matched control settings with shared frozen encoders, local predictors, and readout update budgets, the 21 training-defined low-activity sources yield whole-path Energy ratios for Full relative to No-$D$, history-direct, and direct-$F$ of 0.9997, 0.9981, and 0.9989, with all 90\% bootstrap intervals inside $[0.99,1.01]$.

Across 227 complete control settings, Full improves both whole- and late-horizon mean Energy over all three controls in 60 settings (42 groups), 51 on the shared difficult decile, and 33 on the unaveraged maximum. Rollout stability and tail gains are substantially larger than mean shifts alone. Across three $H=64$ configurations, R\"ossler cuts late Energy by 9.76\% (No-$D$), 7.46\% (history-direct), and 11.29\% (direct-$F$), between-configuration variance by 26.07\%, 20.21\%, and 29.82\%, worst-decile loss by 8.96--18.10\%, and coarse second-half growth slope by 14.96--17.40\%. Similarly, Sprott A cuts late Energy by 10.03\%, 6.84\%, and 1.52\% with medium origin variance down 29.70--56.26\%. On Hourly Solar ($H=48$), keeping $D$ cuts all-origin late Energy vs No-$D$ by 19.53\%, variance by 20.62\%, and a transition-drop origin by 47.52\%.

For the stability analysis, we use all 197 multi-step configurations; the other 30 complete control settings have \(H=1\), so their error-growth slopes are undefined. We measure loss variance across origins, mean error-growth slope, and variance of that slope across origins, each over the whole horizon and its second half. A configuration reaches the stated benefit threshold when at least one of these measures improves by that amount relative to at least one control. Each configuration is counted once.

Training-window variation \(A\) is the mean within-window standard deviation after training-scale normalization. We average across eligible coordinates, up to eight equally spaced windows of \(H+2\) observations per segment, and up to four equally weighted training segments. Coordinates require at least three observed values. Stratifying the 197 configurations by \(A\) gives exploratory groups for comparing stability benefits. Reaching $\ge 5\%$ stability improvement rises monotonically from 6/58 (10.3\%) for $A<0.25$ and 27/83 (32.5\%) for $0.25\le A<0.75$ to 45/56 (80.4\%, with 35/56 or 62.5\% reaching 10\%) for $A\ge 0.75$, where predicted $D$ beats zero change in 42/56 cases vs 9/58. We also examine the validation diagnostic \(S=bv/\sqrt{v+e}\), where \(b\) is the pair-Brier loss under zero mean \(D\) minus that under Full, \(v\) is the training-pair-frequency-weighted variance of local conditional means evaluated on validation histories, and \(e\) is their one-step MSE on observed validation pairs. In this exploratory analysis, \(S\) correlates with signed No-$D$-minus-Full second-half Energy gain ($r=0.710$, $[0.346,0.915]$) and span error growth ($r=0.797$, $[0.203,0.953]$) across 195 configs (121 sources), cutting MSE by 16.73\% and 15.75\% even when holding out whole families.

Targeted stage interventions confirm each upstream stage. Over 1,190 origins, keeping predicted mean $D$ and its time order vs zeroed/permuted $D$ lowers whole Energy on coarse R\"ossler (57.81\% / 15.00\%, variance 96.63\%), coarse Sprott A (18.98\% / 18.32\%), Hourly Solar (36.08\% / 34.38\%, variance 89.36\%), and LOOP Seattle (4.13\% / 5.19\%), matching validation Brier benefit ($\rho=0.800$, $q=0.0051$). On known operators ($d_t=Vz_t$), STRIDE achieves 0.0784 action error vs 0.1950 without coupling and 0.3049 without memory. Dropping memory raises $D$ MSE on Aizawa ($0.005016 \to 0.079724$) and StocBench ($0.004817 \to 0.007411$, 20.20\% below zero change in Figure~\ref{fig:active-evidence}a), while full-history memory initialization lowers $D$ MSE by 1.86--8.30\% on 7/8 added datasets and cuts Energy on coarse R\"ossler (8.75\%, variance 61.13\%) and Mackey--Glass (16.70\%, variance 46.45\%). On 64 Aizawa trajectories, recovering $F$ from $D$ cuts MSE 76.05\% vs direct-$F$ over leads 513--1024.

\begin{figure}[t]
\centering\includegraphics[width=0.95\linewidth]{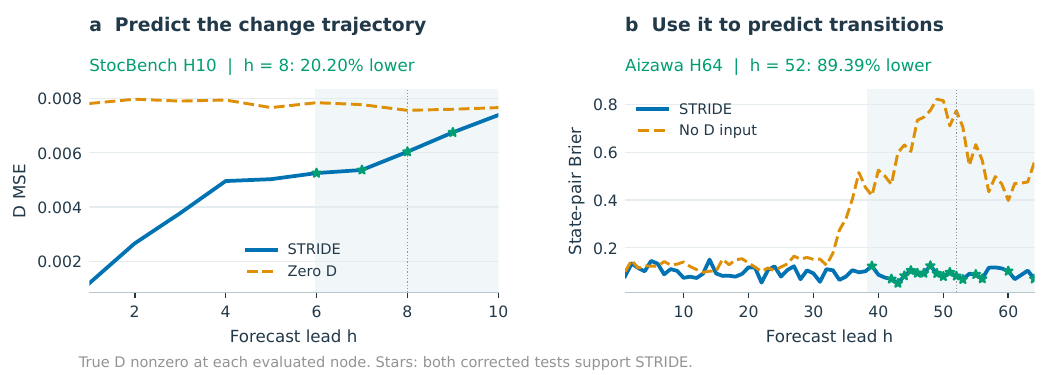}
\caption{\textbf{Increment prediction and its use in transition decoding.} Each lead averages origins with true \(D\ne0\), across three fits. Left: StocBench \(D\) prediction versus zero change. Right: Aizawa joint-transition prediction versus matched No-\(D\) decoding. Stars require both the simultaneous band and block sign-flip test to pass the original full-family correction. The common 80\% lead is marked; shading marks the analyzed final 40\%.}
\label{fig:active-evidence}
\end{figure}

\subsection{Causal deconvolution, local routing, and calibration}
\label{sec:inverse-results}

Downstream tests verify causal deconvolution, local routing, and calibration. On Aizawa $H=64$, frequency inputs cut active-node pair Brier 89.39\% at the 80\% lead vs No-$D$ (Figure~\ref{fig:active-evidence}b), and oracle true-$D/F$ inputs reach 95.21\% ($H=64$) and 92.02\% ($H=256$) accuracy. Across the inventory, causal deconvolution lowers decoded-window MSE against $\widehat F^{\rm raw}$ in 185/227 configs (38.42\% median; 155/197 multi-step, 40.12\% median) and beats projected algebraic inversion on StocBench ($0.730442 \to 0.713792$, 2.28\%) and all eight added datasets (0.56--34.14\% Energy, 7.53--56.53\% pair Brier, up to 34.14\% on DoubleGyre), where small gaps on PAMAP2 (0.56\%) and SleepEDF (0.83\%) stem from low local-mean divergence (variance/MSE 3.91\% and 3.05\%).

Across 197 multi-step settings, keeping each state pair with its fitted local mean lowered whole-horizon Energy by 14.77\% compared with shuffling the means across pairs (95\% source-bootstrap interval, 7.39--23.46\%). The reduction was 43.02\% on coarsely sampled Mackey--Glass. On Aizawa, teacher-forced RMSE was $1.11\times10^{-4}$ with the local functions and $1.58\times10^{-4}$ with the shared predictor. Adding second-order pair dependence leaves Energy virtually unchanged (ratio 0.999972) with upstream memory encoding observed history, while forcing post-hoc re-encoding agreement harms rollouts (+60.02\% Weather, +6.55\% METR-LA). Validation spread rescaling resolves raw 90\% under-coverage, lifting Aizawa coverage from 62.73\% to 85.32\% ($-6.58\%$ Energy) and reaching 90.14\%, 90.67\%, and 88.80\% on Exchange, StocBench, and GARCH with lower Energy and fair CRPS.

\subsection{Matched long-horizon contribution}
\label{sec:long-horizon-control}

When context fades on Aizawa $H=256$, matched fits yield late Energy 0.090716 vs 0.221461 for No-$D$ (59.0\%, Holm $p=0.4375$), and the locked 64-trajectory study reaches 0.090483 vs 0.186657 (No-$D$, 51.5\%), 0.203570 (Same-$q$ Markov, 55.6\%), 0.188623 (history-direct, 52.0\%), and 0.119340 (duration, 24.2\%), all passing Holm correction ($p=0.0004$).

\section{Discussion and conclusion}
\label{sec:discussion}

By tracking windowed transition change rates in POD coordinates, evolving coupled sub-trajectories with memory, and decoding stepwise rules via causal probability-flow deconvolution, \LL\ separates regime drift from static baselines while staged fitting isolates every module.

Across 238 benchmark settings, 20 switching comparisons, and 243 controls, tracking transition evolution improves accuracy and rollout stability where regimes shift and local dynamics diverge (local routing: 14.77\% Energy reduction; deconvolution improving over raw windows in 185/227 configs and algebraic inversion on all eight added datasets), while low-activity sources stay within $\pm 1\%$ of controls. In the high-variation group ($A\ge 0.75$), STRIDE improved at least one stability measure by 5\% or more over at least one control in 45 of 56 settings (80.4\%). Predictable frequency drift, distinct local predictors, sufficient history, and training coverage (confirmed on Allee and KS) define STRIDE's scope alongside calibrated 90\% intervals (85.32--90.67\%).

\label{main-text-end}
\section*{Reproducibility statement}
The benchmark comparisons use saved predictions or frozen model checkpoints. We retain the implementation, task registry, environment inventory, and table and figure generators with the experimental records. The independent Aizawa archive contains simulation and fitting scripts, selected fits, all 960 score records, and 45 fixed-index prediction batches supporting replay of 360 records. The switching comparisons retain matched task scores, selected-model and prediction hashes, and the table-generation scripts.

\section*{Ethics statement}
The study evaluates forecasting methods on existing benchmark data and numerical simulations. Financial series are used solely for forecasting evaluation. Deployment requires validation of forecast coverage and numerical reliability in the intended application.

\section*{AI use statement}
Generative AI tools assisted with literature retrieval, figure preparation, and manuscript drafting and editing. The authors take responsibility for the final manuscript, including its references, figures, and scientific claims.

\clearpage
\begingroup\raggedright\urlstyle{same}\bibliography{references}\bibliographystyle{iclr2027_conference}\endgroup

\end{document}